\documentclass[twoside,11pt]{article}
\usepackage{subfig,nth}
\usepackage{booktabs}
\usepackage{times,epsfig,wrapfig,amsfonts,bm,color,enumitem,algorithm,algpseudocode}
\usepackage{amsmath,amssymb}

\newcommand{\inner}[2]{\left\langle #1, #2 \right\rangle}

\newcommand{\norm}[1]{\left\lVert{#1}\right\rVert}
\newcommand{\abs}[1]{\left\lvert{#1}\right\rvert}

\newcommand{\R}{\mathbb{R}}

\newcommand{\calS}{\mathcal{S}}

\newtheorem{thm}{Theorem}[section]
\newtheorem{lem}{Lemma}[section]

\newcounter{saveenumi}

\algrenewcommand\algorithmicrequire{\textbf{Parameters:}}
\algrenewcommand\algorithmicensure{\textbf{Initialization:}}

\newcommand{\removed}[1]{}

\usepackage{enumitem}
\usepackage{url}

\newcommand{\vx}{\boldsymbol{x}}

\newcommand{\vin}{v_{in}}
\newcommand{\vout}{v_{out}}
\newcommand{\relu}{\sigma_{\textsc{relu}}}

\newcommand{\IN}{\text{in}}
\newcommand{\OUT}{\text{out}}

\newcommand{\picwidth}{1.84in}

\usepackage{jmlr2e}

\ShortHeadings{Geometry of Optimization and Implicit Regularization Deep Learning}{Neyshabur et al.}
\firstpageno{1}

\begin{document}

\title{Geometry of Optimization and Implicit Regularization\\ in Deep Learning}
\author{\\\addr This survey chapter was done as a part of Intel Collaborative Research institute for Computational Intelligence (ICRI-CI) ``Why \& When Deep Learning works -- looking inside Deep Learning'' compendium with the generous support of ICRI-CI.\\ \\\\
\name Behnam Neyshabur \email bneyshabur@ttic.edu \\
       \addr Toyota Technological Institute at Chicago\\
       Chicago, IL 60637, USA
       \AND
       \name Ryota Tomioka \email ryoto@microsoft.com \\
       \addr MSR Cambridge\\
       Cambridge, CB1 2FB, UK
       \AND
       \name Ruslan Salakhutdinov \email rsalakhu@cs.cmu.edu \\
       \addr School of Computer Science\\
       Carnegie Mellon University\\
       Pittsburgh, PA 15213, USA
       \AND
       \name Nathan Srebro \email nati@ttic.edu \\
       \addr  \addr Toyota Technological Institute at Chicago\\
       Chicago, IL 60637, USA}
\editor{}

\maketitle

\begin{abstract}
  We argue that the optimization plays a crucial role in
  generalization of deep learning models through implicit
  regularization.  We do this by demonstrating that generalization
  ability is not controlled by network size but rather by some other
  implicit control.  We then demonstrate how changing the empirical
  optimization procedure can improve generalization, even if actual
  optimization quality is not affected.  We do so by studying the
  geometry of the parameter space of deep networks, and devising an
  optimization algorithm attuned to this geometry. 
\end{abstract}


\section{Introduction}
Central to any form of learning is an inductive bias that induces some
sort of capacity control (i.e.~restricts or encourages predictors to
be ``simple'' in some way), which in turn allows for generalization.
The success of learning then depends on how well the inductive bias
captures reality (i.e.~how expressive is the hypothesis class of
``simple'' predictors) relative to the capacity induced, as well as on
the computational complexity of fitting a ``simple'' predictor to the
training data.

Let us consider learning with feed-forward networks from this
perspective.  If we search for the weights minimizing the training
error, we are essentially considering the hypothesis class of
predictors representable with different weight vectors, typically for
some fixed architecture.  Capacity is then controlled by the size
(number of weights) of the network\footnote{The exact correspondence
  depends on the activation function---for hard thresholding
  activation the pseudo-dimension, and hence sample complexity, scales
  as $O(S \log S)$, where $S$ is the number of weights in the network. 
  With sigmoidal activation it is between $\Omega(S^2)$ and
  $O(S^4)$ \citep{anthony1999}.}.  Our justification for using such networks is
then that many interesting and realistic functions can be represented
by not-too-large (and hence bounded capacity) feed-forward networks.
Indeed, in many cases we can show how specific architectures can
capture desired behaviors.  More broadly, any $O(T)$ time computable
function can be captured by an $O(T^2)$ sized network, and so the
expressive power of such networks is indeed great~\cite[Theorem~9.25]{sipser2012}.

At the same time, we also know that learning even moderately sized
networks is computationally intractable---not only is it NP-hard to
minimize the empirical error, even with only three hidden units, but
it is hard to learn small feed-forward networks using {\em any}
learning method (subject to cryptographic assumptions).  That is, even
for binary classification using a network with a single hidden layer and a logarithmic (in the
input size) number of hidden units, and even if we know the true
targets are {\em exactly} captured by such a small network, there is
likely no efficient algorithm that can ensure error better than 1/2
\citep{klivans2006cryptographic,Daniely14}---not if the algorithm
tries to fit such a network, not even if it tries to fit a much larger
network, and in fact no matter how the algorithm represents
predictors.  And so, merely knowing that some not-too-large
architecture is excellent in expressing reality does {\em not} explain
why we are able to learn using it, nor using an even larger network.
Why is it then that we succeed in learning using multilayer
feed-forward networks?  Can we identify a property that makes them
possible to learn?  An alternative inductive bias?

In section \ref{sec:netsize}, we make our first steps at shedding light on this question by
going back to our understanding of network size as the capacity
control at play. Our main observation, based on empirical experimentation with
single-hidden-layer networks of increasing size (increasing number of
hidden units), is that size does {\em not} behave as a capacity
control parameter, and in fact there must be some other, implicit,
capacity control at play.  We suggest that this hidden capacity
control might be the real inductive bias when learning with deep
networks.

Revisiting the choice of gradient descent, we recall that optimization is inherently tied to a choice of geometry
or measure of distance, norm or divergence. Gradient descent for example is tied to the $\ell_2$ norm as it
is the steepest descent with respect to $\ell_2$ norm in the parameter space, while coordinate descent corresponds
to steepest descent with respect to the $\ell_1$ norm and exp-gradient (multiplicative weight) updates is tied to
an entropic divergence. Moreover, at least when the objective function is convex, convergence behavior is
tied to the corresponding norms or potentials. For example, with gradient descent, or SGD, convergence
speeds depend on the $\ell_2$ norm of the optimum.  The norm or divergence can be viewed as a
regularizer for the updates.  There is therefore also a strong link
between regularization for optimization and regularization for
learning: optimization may provide implicit regularization in terms of
its corresponding geometry, and for ideal optimization performance the
optimization geometry should be aligned with inductive bias driving
the learning \citep{srebro11}.

Is the $\ell_2$ geometry on the weights the appropriate geometry for
the space of deep networks?  Or can we suggest a geometry with more
desirable properties that would enable faster optimization and perhaps
also better implicit regularization?  As suggested above, this
question is also linked to the choice of an appropriate regularizer
for deep networks.

Focusing on networks with RELU activations in this section, we observe that scaling down the incoming edges to a hidden unit and scaling up the outgoing edges by the same factor yields an equivalent network computing the same function. Since predictions are invariant to such rescalings, it is natural to seek a geometry, and corresponding optimization method, that is similarly invariant.

We consider here a geometry inspired by max-norm regularization
(regularizing the maximum norm of incoming weights into any unit)
which seems to provide a better inductive bias compared to the
$\ell_2$ norm (weight decay) \citep{goodfellow13,srivastava14}.  But to
achieve rescaling invariance, we use not the max-norm itself, but
rather the minimum max-norm over all rescalings of the weights.  We
discuss how this measure can be expressed as a ``path regularizer''
and can be computed efficiently.

We therefore suggest a novel optimization method, Path-SGD, that is an
approximate steepest descent method with respect to path
regularization.  Path-SGD is rescaling-invariant and we demonstrate that
Path-SGD outperforms gradient descent and AdaGrad for classifications
tasks on several benchmark datasets. This again demonstrates the importance of
implicit regularization that is introduced by optimization.

This summary paper combines material previously presented by the
authors at the \nth{3} International Conference on Learning
Representations (ICLR), the \nth{28} Conference on Learning Theory
(COLT) and Advances in Neural Information Processing Systems (NIPS)
28, as well as Intel Collaborative Research Institutes retreats
\citep{neyshabur2015path,neyshabur2015norm,neyshabur2015search}.

\paragraph{Notations} A feedforward neural network that computes a
function $f:\R^D \rightarrow \R^C$ can be represented by a directed
acyclic graph (DAG) $G(V,E)$ with $D$ input nodes
$\vin[1],\dots,\vin[D]\in V$, $C$ output nodes $\vout[1],\dots,
\vout[C]\in V$, weights $w:E\rightarrow \R$ and an activation function
$\sigma:\R\rightarrow\R$ that is applied on the internal nodes (hidden
units). We denote the function computed by this network as
$f_{G,w,\sigma}$. In this paper we focus on RELU (REctified Linear
Unit) activation function $\relu(x) = \max\{0,x\}$. We refer to the
depth $d$ of the network which is the length of the longest directed
path in $G$.  For any $0\leq i \leq d$, we define $V^i_{\IN}$ to be
the set of vertices with longest path of length $i$ to an input unit
and $V^{i}_{\OUT}$ is defined similarly for paths to output units. In
layered networks $V^i_{\IN} = V^{d-i}_{\OUT}$ is the set of hidden
units in a hidden layer $i$.

\section{Implicit Regularization}\label{sec:netsize}

\begin{figure*}[t!]
\hbox{ \centering
\includegraphics[width=0.455\textwidth]{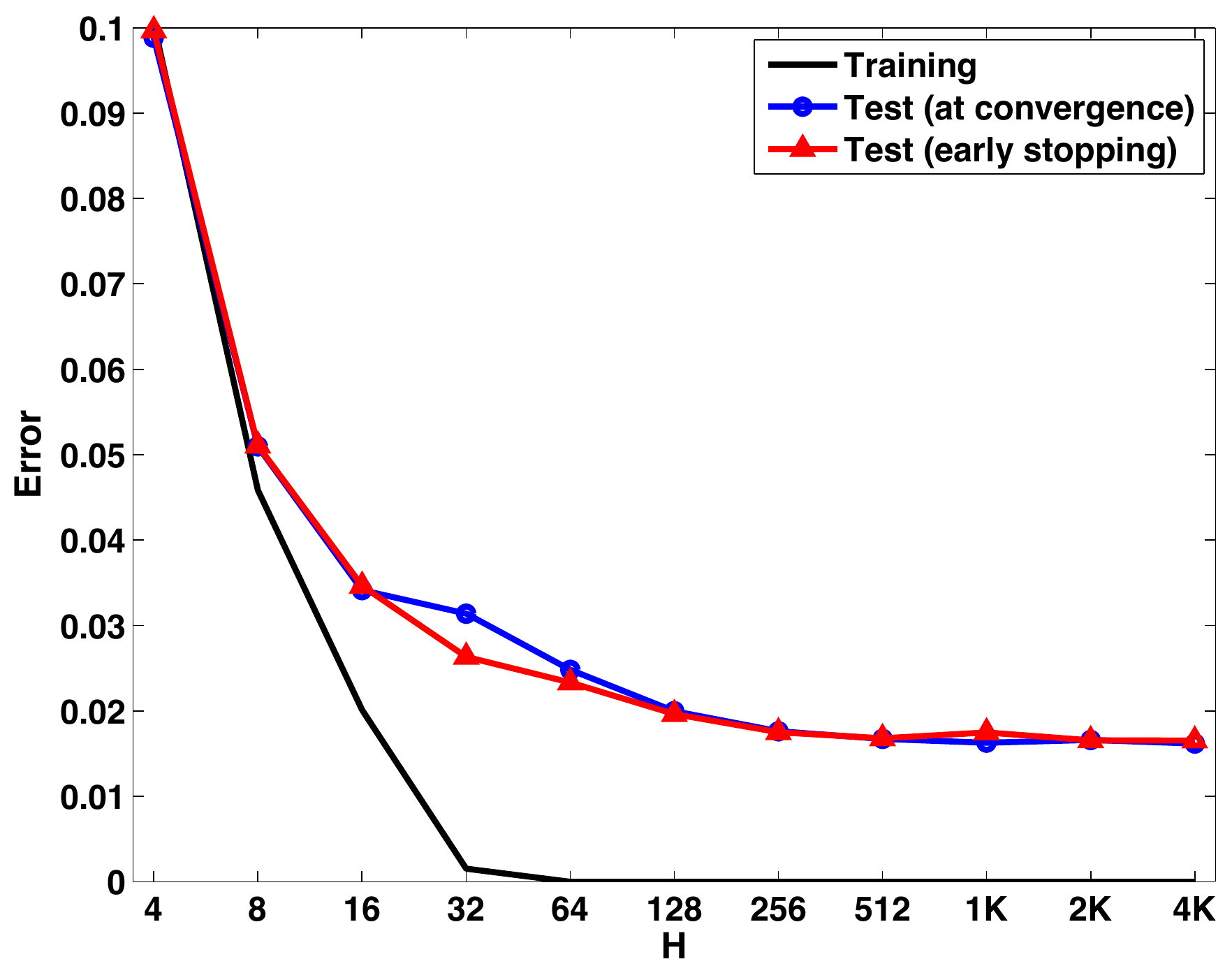}
\hspace{0.1in}
\includegraphics[width=0.45\textwidth]{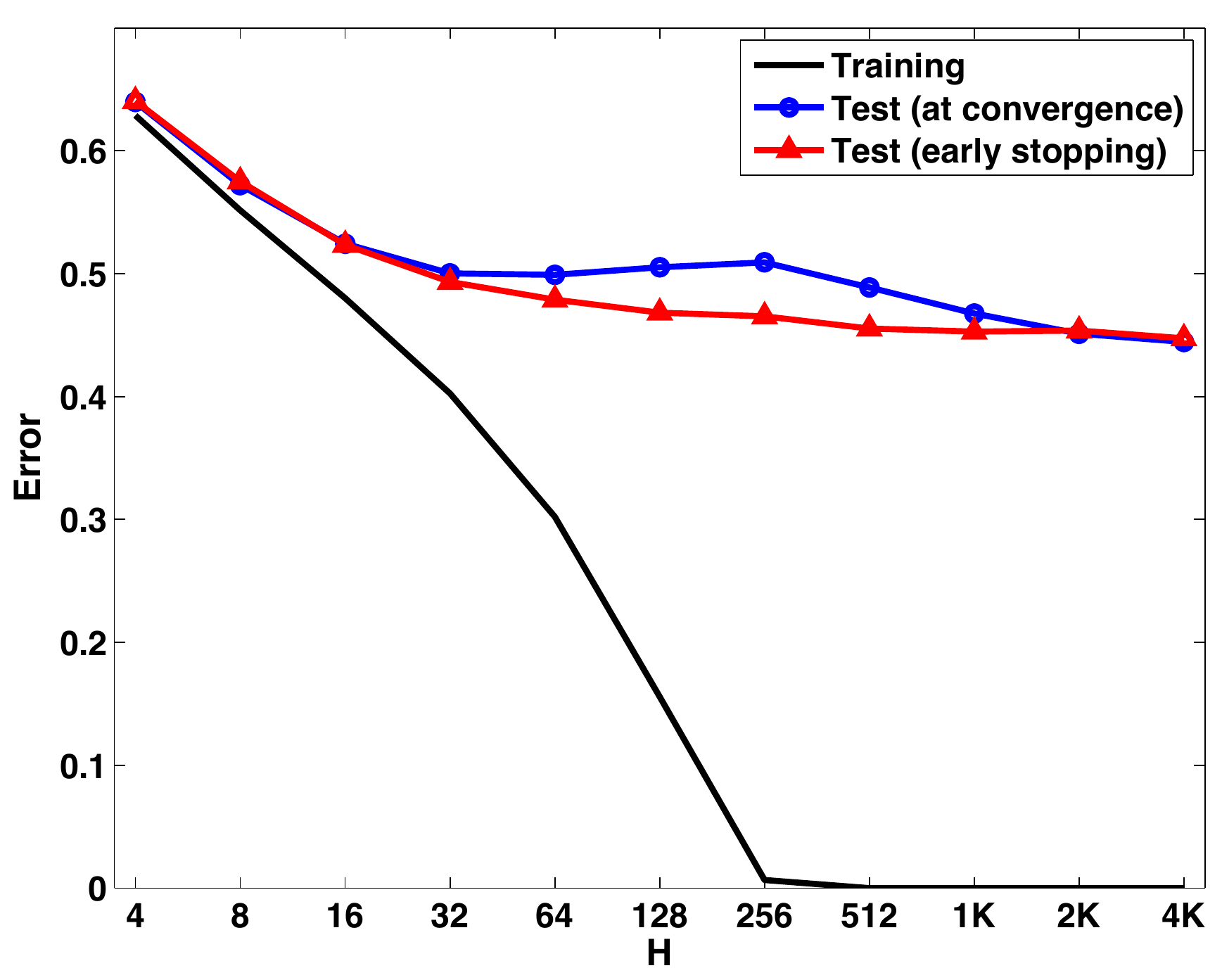}
}
\begin{picture}(0,0)(0,0)
\put(85, 155){MNIST}
\put(280, 155){CIFAR-10}
\end{picture}
\vspace{0.1in}
\caption{\footnotesize The training error and the test error based on different stopping
  criteria when 2-layer NNs with different number of hidden units
  are trained on MNIST and CIFAR-10. Images in both datasets are downsampled
  to 100 pixels. The size of the training set is 50000 for MNIST and 40000 for CIFAR-10.
  The early stopping is based on the error on a validation set
  (separate from the training and test sets) of size 10000. The training was
  done using stochastic gradient descent with momentum and mini-batches
  of size 100. The network was initialized with weights generated randomly from
  the Gaussian distribution. The initial step size and momentum were set to 0.1 and 0.5
  respectively. After each epoch, we used the update rule $\mu^{(t+1)}=0.99\mu^{(t)}$
  for the step size and $m^{(t+1)}=\min\{0.9,m^{(t)}+0.02\}$ for the momentum. 
  \label{fig:intro}
}
\end{figure*}
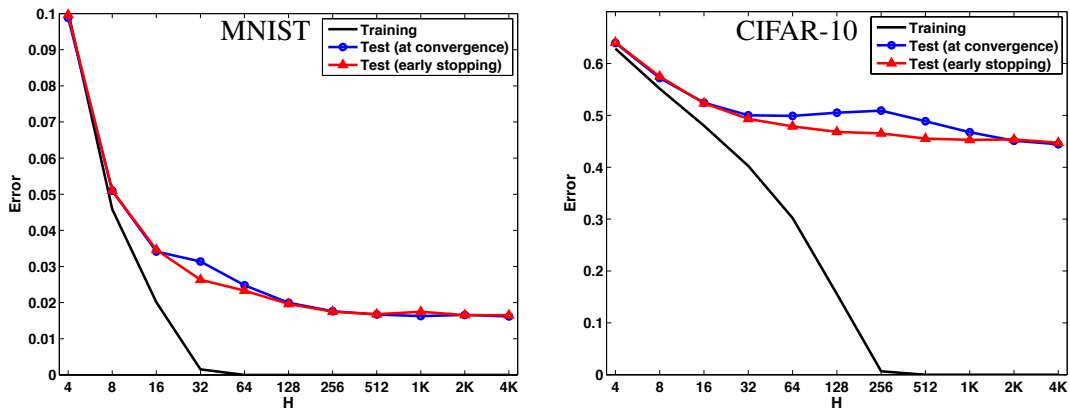

Consider training a feed-forward network by finding the weights
minimizing the training error.  Specifically, we will consider a
network with $D$ real-valued inputs $\vx=(x[1],\ldots,x[D])$, a single
hidden layer with $H$ rectified linear units, and $C$ outputs
$y[1],\ldots,y[k]$ where the weights are learned by minimizing
a (truncated) soft-max cross entropy loss\footnote{When using soft-max
  cross-entropy, the loss is never exactly zero for correct
  predictions with finite margins/confidences.  Instead, if the data
  is separable, in order to minimize the loss the weights need to be
  scaled up toward infinity and the cross entropy loss goes to zero,
  and a global minimum is never attained.  In order to be able to say
  that we are actually reaching a zero loss solution, and hence a
  global minimum, we use a slightly modified soft-max which does not
  noticeably change the results in practice.  This truncated loss
  returns the same exact value for wrong predictions or correct
  prediction with confidences less than a threshold but returns zero
  for correct predictions with large enough margins: Let
  $\{s_i\}_{i=1}^k$ be the scores for $k$ possible labels and $c$ be
  the correct labels. Then the soft-max cross-entropy loss can be
  written as $\ell(s,c) = \ln \sum_{i} \exp(s_i - s_c)$ but we instead
  use the differentiable loss function $\hat{\ell}(s,c) = \ln \sum_{i}
  f(s_i-s_c)$ where $f(x)=\exp(x)$ for $x\geq -11$ and $f(x)
  =\exp(-11) [x+13]_+^2/4$ otherwise. Therefore, we only deviate from
  the soft-max cross-entropy when the margin is more than $11$, at
  which point the effect of this deviation is negligible (we always
  have $\abs{\ell(s,c)-\hat{\ell}(s,c)}\leq 0.000003k$)---if there are
  any actual errors the behavior on them would completely dominate
  correct examples with margin over $11$, and if there are no errors
  we are just capping the amount by which we need to scale up the
  weights.} on $n$ labeled training examples.  The total number of
weights is then $H(C+D)$.

What happens to the training and test errors when we increase the
network size $H$? The training error will necessarily decrease.  The
test error might initially decrease as the approximation error is
reduced and the network is better able to capture the targets.
However, as the size increases further, we loose our capacity control
and generalization ability, and should start overfitting.  This is the
classic approximation-estimation tradeoff behavior.

Consider, however, the results shown in Figure \ref{fig:intro}, where
we trained networks of increasing size on the MNIST and CIFAR-10
datasets.  Training was done using stochastic gradient descent with
momentum and diminishing step sizes, on the training error and without
any explicit regularization.  As expected, both training and test
error initially decrease.  More surprising is that if we increase the
size of the network past the size required to achieve zero training
error, the test error continues decreasing!  This behavior is not at
all predicted by, and even contrary to, viewing learning as fitting a
hypothesis class controlled by network size.  For example for MNIST, 32 units
are enough to attain zero training error.  When we allow more units,
the network is not fitting the training data any better, but the
estimation error, and hence the generalization error, should increase
with the increase in capacity.  However, the test error goes down.  In
fact, as we add more and more parameters, even beyond the number
of training examples, the generalization error does not go up.

What is happening here?  A possible explanation is that the
optimization is introducing some implicit regularization.  That is, we are
implicitly trying to find a solution with small ``complexity'', for
some notion of complexity, perhaps norm.  This can explain why we do
not overfit even when the number of parameters is huge.  Furthermore,
increasing the number of units might allow for solutions that actually
have lower ``complexity'', and thus generalization better.  Perhaps an
ideal then would be an infinite network controlled only through this
hidden complexity.

We want to emphasize that we are not including any explicit
regularization, neither as an explicit penalty term nor by modifying
optimization through, e.g., drop-outs, weight decay, or with
one-pass stochastic methods.  We are using a stochastic method, but we
are running it to convergence---we achieve zero surrogate loss and zero training error. In fact, we also tried training using batch conjugate
gradient descent and observed almost identical behavior.  But it seems
that even still, we are not getting to some random global
minimum---indeed for large networks the vast majority of the many
global minima of the training error would horribly overfit.  Instead,
the optimization is directing us toward a ``low complexity'' global
minimum.

We have argued that the implicit regularization is due to the optimization. It is therefore expected that different optimization methods
introduce different implicit regularizations which leads to different generalization properties. In an attempt to find an optimization method with better
generalization properties, we recall that the optimization is also tied to a choice of geometry/distance measure in the parameter space.  We look into the desirable properties of a geometry for neural networks and suggest an optimization algorithm that is tied to that geometry.
\section{The Geometry of Optimization: Rescaling and Unbalanceness} \label{sec:geometry}

\begin{figure}[t!] \label{fig:unbalanced}
\hspace{0.5in}
\subfloat[Training on MNIST]{
  \includegraphics[width=0.23\textwidth]{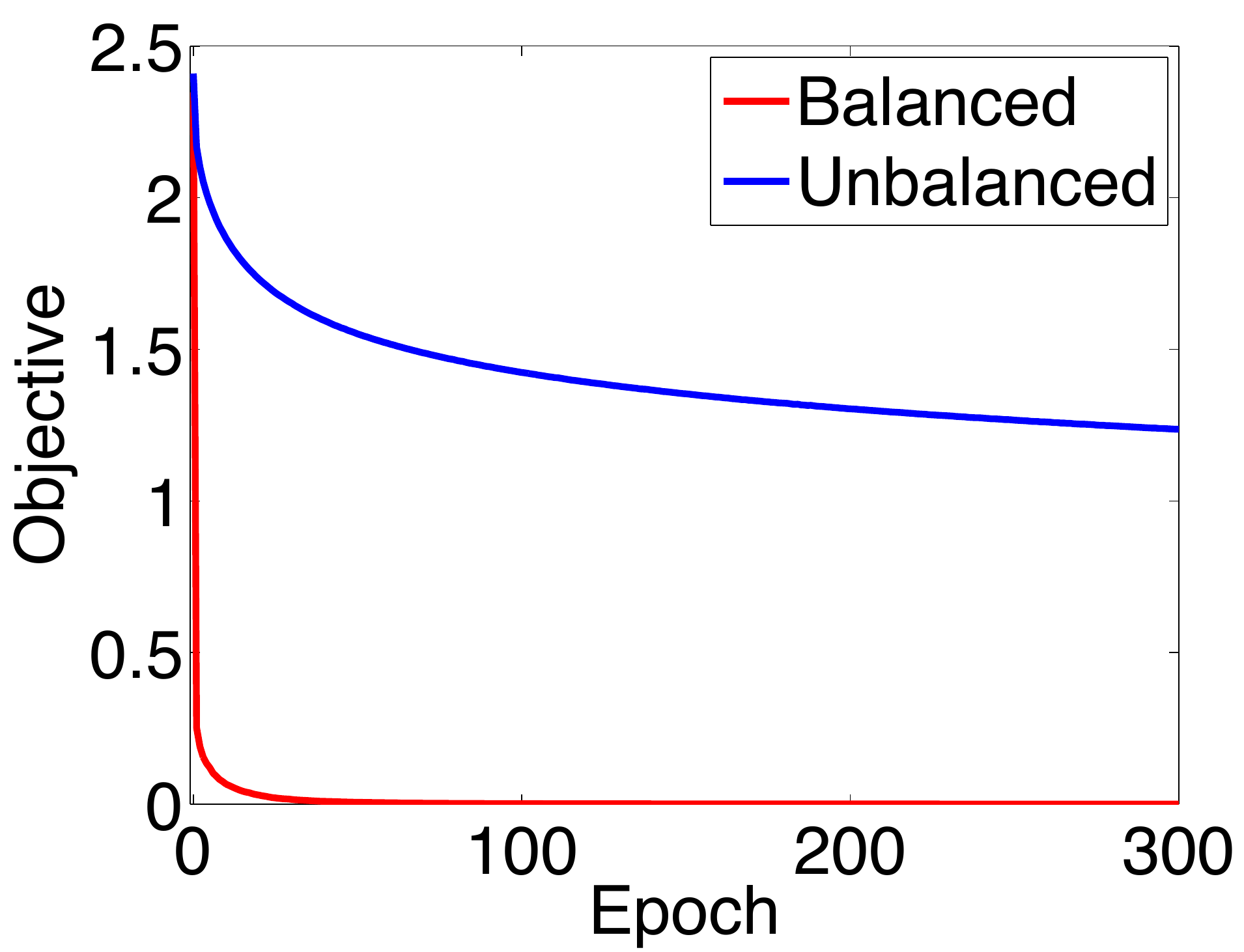}\label{fig:compare-a}
 }\hspace{1in}
 \subfloat[  \small Weight explosion in an unbalanced network]{
 \hspace{0.4in}
  \includegraphics[width=0.32\textwidth]{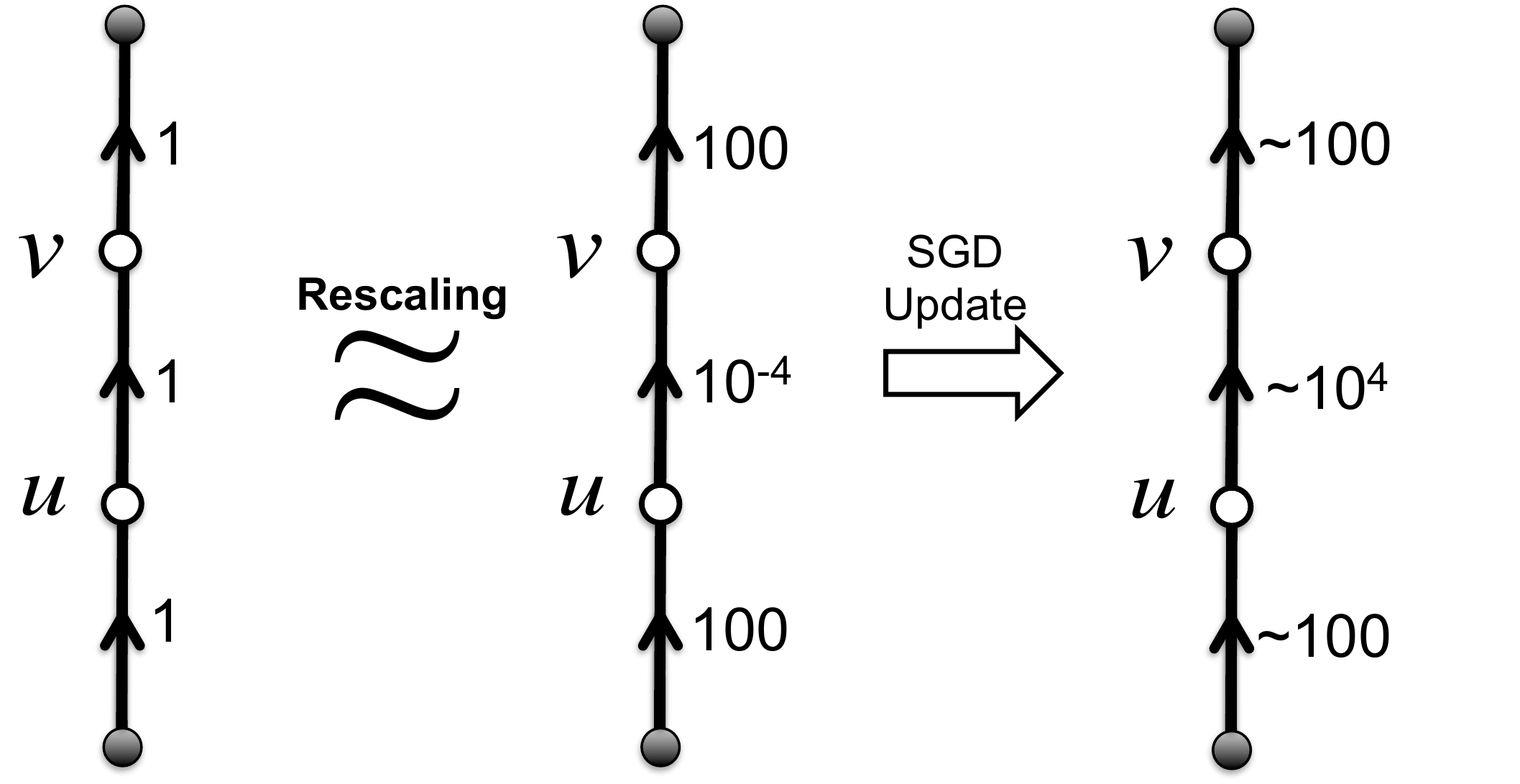}\label{fig:compare-b}
  \hspace{0.3in}
 }
  \vspace{-0.08in}
 \newline
 \begin{center}
  \subfloat[ \small Poor updates in an unbalanced network]{
  \includegraphics[width=0.9\textwidth]{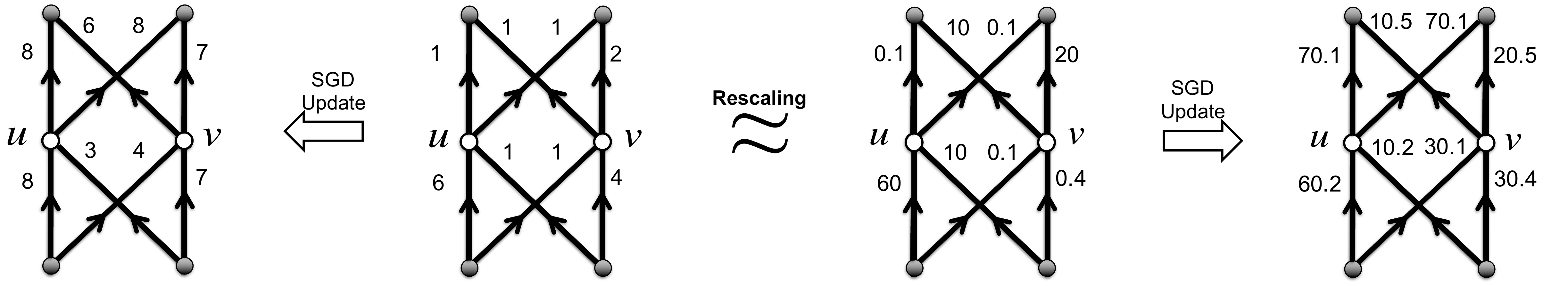}\label{fig:compare-c}
   }
   \end{center}
   \vspace{-0.16in}
 \caption{ \footnotesize (a): Evolution of the cross-entropy error function when training a 
feed-forward network on MNIST with two hidden layers, each containing 4000 hidden units.
The unbalanced initialization (blue curve) is generated by applying a sequence of rescaling functions on the balanced initializations (red curve). (b): Updates for a simple case where the input is $x=1$, 
thresholds are set to zero (constant), the stepsize is 1, and the gradient with respect to output is $\delta = -1$. (c): Updated network for the case where the input is $x=(1,1)$, thresholds are set to zero (constant), the stepsize is 1, and the gradient with respect to output is $\delta=(-1,-1)$. }
 \label{fig:unbalanced}
 \vspace{-0.1in}
\end{figure}

In this section, we look at the behavior of the Euclidean geometry under rescaling and unbalanceness. One of the special properties of RELU activation function is non-negative homogeneity. That is, for any scalar $c\geq 0$ and any $x\in \R$, we have $\relu(c\cdot x)=c\cdot \relu(x)$. This  interesting property allows the network to be rescaled without changing the function computed by the network. We define the {\em rescaling function} $\rho_{c,v}(w)$, such that given the weights of the network $w:E\rightarrow \R$, a constant $c>0$, and a node $v$, the rescaling function multiplies the incoming edges and divides the outgoing edges of $v$ by $c$. That is, $\rho_{c,v}(w)$ maps $w$ to the weights $\tilde{w}$ for the rescaled network, where for any $(u_1\rightarrow u_2)\in E$:
\begin{equation}
\tilde{w}_{(u_1\rightarrow u_2)}=
\begin{cases}
c.w_{(u_1\rightarrow u_2)}& u_2=v,\\
\frac{1}{c}w_{(u_1\rightarrow u_2)}& u_1=v,\\
w_{(u_1\rightarrow u_2)}& \text{otherwise.}\\
\end{cases}
\end{equation}
It is easy to see that the rescaled network computes the same function, i.e. $f_{G,w,\relu} =f_{G,\rho_{c,v}(w),\relu}$. We say that the two networks with weights $w$ and $\tilde{w}$ are {\em rescaling equivalent} denoted by $w\sim \tilde{w}$ if and only if one of them can be transformed to another by applying a sequence of rescaling functions $\rho_{c,v}$.

Given a training set $\calS = \{(x_1,y_n),\dots, (x_n,y_n)\}$, our goal is to 
minimize the following objective function:
\begin{equation}
L(w) = \frac{1}{n}\sum_{i=1}^n \ell(f_w(x_i),y_i).
\end{equation}
Let $w^{(t)}$ be the weights at step $t$ of the optimization. We consider update step of the following form $w^{(t+1)} = w^{(t)} + \Delta w^{(t+1)}$. For example, for gradient descent, we have  $\Delta w^{(t+1)} = -\eta\nabla L(w^{(t)})$, where $\eta$ is the step-size. In the stochastic setting, such as SGD or mini-batch gradient descent, we calculate the gradient on a small subset of the training set. 

Since {\em rescaling equivalent} networks compute the same function,
it is desirable to have an update rule that is not affected by
rescaling. We call an optimization method {\em rescaling invariant} if
the updates of rescaling equivalent networks are rescaling equivalent.
That is, if we start at either one of the two rescaling equivalent
weight vectors $\tilde{w}^{(0)}\sim w^{(0)}$, after applying $t$
update steps separately on $\tilde{w}^{(0)}$ and $w^{(0)}$, they will
remain rescaling equivalent and we have $\tilde{w}^{(t)} \sim
w^{(t)}$.

Unfortunately, gradient descent is {\em not} rescaling invariant. The main problem with the gradient updates is that scaling down the weights of an edge will also scale up the gradient which, as we see later, is exactly the opposite of what is expected from a rescaling invariant update. 

Furthermore, gradient descent performs very poorly on ``unbalanced''
networks.  We say that a network is {\em balanced} if the norm of
incoming weights to different units are roughly the same or within a
small range. For example, Figure~\subref*{fig:compare-a} shows a huge
gap in the performance of SGD initialized with a randomly generated
balanced network $w^{(0)}$, when training on MNIST, compared to a
network initialized with unbalanced weights $\tilde{w}^{(0)}$. Here
$\tilde{w}^{(0)}$ is generated by applying a sequence of random
rescaling functions on $w^{(0)}$ (and therefore $w^{(0)}\sim
\tilde{w}^{(0)}$).

In an unbalanced network, gradient descent updates could blow up the smaller weights, while keeping the larger weights almost unchanged. This is illustrated in Figure ~\subref*{fig:compare-b}. If this were the only issue, one could scale down all the weights after each update. However, in an unbalanced network,  the relative changes in the weights are also very different compared to a balanced network. For example, Figure \subref*{fig:compare-c} shows how two rescaling equivalent networks could end up computing a very different function after only a single update.

\removed{The same problem exists in other optimization methods used in deep learning, including AdaGrad updates~\citep{adagrad}, where $\Delta w^{(t+1)}_e = -\eta \left(\partial L / \partial w_e^{(t)}\right)/\sqrt{\sum_{k=1}^t (\partial L / \partial w_e^{(k)})^2}$. In order to better understand and investigate this issue, we next discuss different scale measures for training neural networks.}

\section{Magnitude/Scale measures for deep networks}
Following \cite{neyshabur2015norm}, we consider the grouping of weights going into each node of the
network. This forms the following generic group-norm type regularizer, parametrized by $1\leq p,q \leq\infty$:
\begin{equation}
  \label{eq:mu}
  \mu_{p,q}(w) = \left(\sum_{v \in V}\left(\sum_{(u\rightarrow v) \in E} \left\lvert w_{(u\rightarrow v)}\right\rvert ^p\right)^{q/p}\right)^{1/q}.
\end{equation}
Two simple cases of above group-norm are $p=q=1$ and $p=q=2$ that
correspond to overall $\ell_1$ regularization and weight decay
respectively. Another form of regularization that is shown to be very
effective in RELU networks is the max-norm regularization, which is the
maximum over all units of norm of incoming edge to the
unit\footnote{This definition of max-norm is a bit different than the
  one used in the context of matrix factorization~\citep{srebro05}. The
  later is similar to the minimum upper bound over $\ell_2$ norm of
  both outgoing edges from the input units and incoming edges to the
  output units in a two layer feed-forward
  network.}~\citep{goodfellow13,srivastava14}. The max-norm correspond
to ``per-unit" regularization when we set $q=\infty$ in
equation~\eqref{eq:mu} and can be written in the following form:
\begin{equation}
  \label{eq:mu}
  \mu_{p,\infty}(w) =\sup_{v \in V}\left(\sum_{(u\rightarrow v) \in E} \left\lvert w_{(u\rightarrow v)}\right\rvert ^p\right)^{1/p}
\end{equation}

Weight decay is probably the most commonly used regularizer. On the
other hand, per-unit regularization might not seem ideal as it is
very extreme in the sense that the value of regularizer corresponds to
the highest value among all nodes.  However, the situation is very
different for networks with RELU activations (and other activation
functions with non-negative homogeneity property).  In these cases,
per-unit $\ell_2$ regularization has shown to be very
effective~\citep{srivastava14}. The main reason could be because RELU
networks can be rebalanced in such a way that all hidden units have
the same norm. Hence, per-unit regularization will not be a crude
measure anymore.

Since $\mu_{p,\infty}$ is not rescaling invariant and the values of the
scale measure are different for rescaling equivalent networks, it is
desirable to look for the minimum value of a regularizer among all
rescaling equivalent networks. Surprisingly, for a feed-forward
network, the minimum $\ell_p$ per-unit regularizer among all rescaling
equivalent networks can be efficiently computed by a single forward
step. To see this, we consider the vector $\pi(w)$, the {\em path
  vector}, where the number of coordinates of $\pi(w)$ is equal to the
total number of paths from the input to output units and each
coordinate of $\pi(w)$ is the equal to the product of weights along a
path from an input nodes to an output node. The $\ell_p$-path
regularizer is then defined as the $\ell_p$ norm of
$\pi(w)$~\citep{neyshabur2015norm}:
\begin{equation}\label{eq:defphi}
  \phi_p(w) = \norm{\pi(w)}_p = \left(\sum_{\vin[i] \overset{e_1}\rightarrow v_1\overset{e_2}\rightarrow v_2\dots\overset{e_d}{\rightarrow}\vout[j]} \left|\prod_{k=1}^d w_{e_k}\right|^p\right)^{1/p}
\end{equation}
The following Lemma establishes that the $\ell_p$-path regularizer
corresponds to the minimum over all equivalent networks of the
per-unit $\ell_p$ norm:
\begin{lem}[\cite{neyshabur2015norm}]\label{lem:path-unit}
$\displaystyle \phi_p(w) = \min_{\tilde{w} \sim w} \bigg(\mu_{p,\infty}(\tilde{w})\bigg)^d$
\end{lem}
The definition \eqref{eq:defphi} of the $\ell_p$-path regularizer
involves an exponential number of terms.  But it can be computed
efficiently by dynamic programming in a single forward step using the
following equivalent form as nested sums:
\begin{equation*}
\phi_p(w) = \left(\sum_{(v_{d-1}\rightarrow v_{\OUT}[j])\in E}\left| w_{(v_{d-1}\rightarrow v_{\OUT}[j])}\right|^p\sum_{(v_{d-2}\rightarrow v_{d-1})\in E}\dots \sum_{(v_{\IN}[i]\rightarrow v_{1})\in E} \left| w_{(v_{\IN}[i]\rightarrow v_{1})}\right|^p \right)^{1/p}
\end{equation*}
A straightforward consequence of Lemma \ref{lem:path-unit} is that the $\ell_p$ path-regularizer $\phi_p$ is invariant to rescaling, i.e. for any $\tilde{w} \sim w$, $\phi_p(\tilde{w})=\phi_p(w)$.

\section{Path-SGD: An Approximate Path-Regularized Steepest Descent}
Motivated by empirical performance of max-norm regularization and the
fact that path-regularizer is invariant to rescaling, we are
interested in deriving the steepest descent direction with respect to
the path regularizer $\phi_p(w)$:
\begin{align}\label{eq:sd}
w^{(t+1)} &= \arg\min_w \;\;\eta \inner{ \nabla L(w^{(t)})}{w} + \frac{1}{2}\norm{\pi(w)-\pi(w^{(t)})}_p^2\\ \notag
&= \arg\min_w \;\;\eta \inner{ \nabla L(w^{(t)}) }{w} + \left(\sum_{\vin[i] \overset{e_1}\rightarrow v_1\overset{e_2}\rightarrow v_2\dots\overset{e_d}{\rightarrow}\vout[j]}\left(\prod_{k=1}^d w_{e_k} - \prod_{k=1}^d w^{(t)}_{e_k})\right)^p\right)^{2/p}\\ \notag
& = \arg\min_w J^{(t)}(w)
\end{align}
The steepest descent step \eqref{eq:sd} is hard to calculate exactly.  Instead, we will update each coordinate $w_e$ independently (and synchronously) based on~\eqref{eq:sd}. That is:
\begin{equation}
w^{(t+1)}_e =\arg\min_{w_e} \;J^{(t)}(w) \qquad \text{s.t.}\;\;\forall_{e'\neq e} \;\;w_{e'}=w^{(t)}_{e'}
\end{equation}
Taking the partial derivative with respect to $w_e$ and setting it to zero we obtain:
\begin{equation*}
0 =\eta \frac{\partial L}{\partial w_e}(w^{(t)}) - \left(w_e-w^{(t)}_e\right) \left(\sum_{v_{\text{in}}[i] \dots \stackrel{e}{\rightarrow} \dots v_{\text{out}}[j]} \prod_{e_k\neq e} \abs{w^{(t)}_e}^p\right)^{2/p}
\end{equation*}
where $v_{\text{in}}[i] \dots \stackrel{e}{\rightarrow} \dots v_{\text{out}}[j]$ denotes the paths from any input unit $i$ to any output unit $j$ that includes $e$. Solving for $w_e$  gives us the following update rule:
\begin{equation}\label{eq:update}
\hat{w}^{(t+1)}_e = w^{(t)}_e- \frac{\eta}{\gamma_p(w^{(t)},e)}\frac{\partial L}{\partial w}(w^{(t)})
\end{equation}
where $\gamma_p(w,e)$ is given as
\begin{equation}
\gamma_p(w,e) =\left(\sum_{v_{\text{in}}[i] \dots \stackrel{e}{\rightarrow} \dots v_{\text{out}}[j]} \prod_{e_k\neq e} \abs{w_{e_k}}^p\right)^{2/p}
\end{equation}
We call the optimization using the update rule \eqref{eq:update} path-normalized gradient descent. When used in stochastic settings, we refer to it as Path-SGD.

\removed{ In the following lemma, we prove that if the relative change in the weights of the network is small enough (which can be achieved by choosing a small enough stepsize), then the update rule~\eqref{eq:update} is an approximate steepest direction with respect to $\ell_p$-path regularizer.
\begin{lem}
Let $\delta_{\max}$ be the maximum relative change in a weight in the network by update rule~\eqref{eq:update}, i.e.
$\delta_{\max} = \max_{e\in E} \abs{ w^{(t+1)}_e - w^{(t)}_e }/\abs{w^{(t)}_e}$. Then, if $\delta_{\max} \leq \frac{1}{d}$, we have that:
$$
\norm{\pi(w^{(t+1)}) - \pi(w^{(t)})}_p \leq 2d\delta_{\max}\phi_p(w^{(t)})
$$
and so
$$
\phi_p(w^{(t+1)}) \leq 2d\delta_{\max}\phi_p(w^{(t)}).
$$
\end{lem}
\begin{proof}
Consider any path $\vin[i] \overset{e_1}\rightarrow v_1\overset{e_2}\rightarrow v_2\dots\overset{e_d}{\rightarrow}\vout[j]$ in the network. It is clear that:
\begin{equation*}
\abs{ \prod_{k=1}^d w^{(t+1)}_{e_k} - \prod_{k=1}^d w_{e_k}} \leq \left[(1+\delta_{\max})^d -1\right]\abs{\prod_{k=1}^d w_{e_k}} 
\end{equation*}
Applying the above inequality on all paths and taking the sum gives us:
\begin{align*}
\norm{\pi(w^{(t+1)}) - \pi(w^{(t)})}_p  &\leq \left[(1+\delta_{\max})^d-1\right]\phi_p(w^{(t)})\\
&\leq (e-1)d\delta_{\max}
\end{align*}
\end{proof}
}

Now that we know Path-SGD is an approximate steepest descent with respect to the path-regularizer, we can ask whether or not this makes Path-SGD a {\em rescaling invariant} optimization method. The next theorem proves that Path-SGD is indeed rescaling invariant.

\begin{thm}
Path-SGD is rescaling invariant.
\end{thm}
\begin{proof}
It is sufficient to prove that using the update rule~\eqref{eq:update}, for any $c>0$ and any $v\in E$, if $\tilde{w}^{(t)} = \rho_{c,v}(w^{(t)})$, then $\tilde{w}^{(t+1)} = \rho_{c,v}(w^{(t+1)})$. For any edge $e$ in the network, if $e$ is neither incoming nor outgoing edge of the node $v$, then $\tilde{w}(e)=w(e)$, and since the gradient is also the same for edge $e$ we have $\tilde{w}^{(t+1)}_e=w^{(t+1)}_e$. However, if $e$ is an incoming edge to $v$, we have that $\tilde{w}^{(t)}(e)=cw^{(t)}(e)$. Moreover, since the outgoing edges of $v$ are divided by $c$, we get $\gamma_p(\tilde{w}^{(t)},e) = \frac{\gamma_p(w^{(t)},e)}{c^2}$ and $\frac{\partial L}{\partial w_e}(\tilde{w}^{(t)})= \frac{\partial L}{c\partial w_e}(w^{(t)})$. Therefore,
\begin{align*}
{\tilde{w}}^{(t+1)}_e &= cw^{(t)}_e - \frac{c^2\eta}{\gamma_p(w^{(t)},e)} \frac{\partial L}{c\partial w_e}(w^{(t)})\\
&= c\left(w^{(t)} - \frac{\eta}{\gamma_p(w^{(t)},e)} \frac{\partial L}{\partial w_e}(w^{(t)})\right) = cw^{(t+1)}_e.
\end{align*}
A similar argument proves the invariance of Path-SGD update rule for outgoing edges of $v$. Therefore, Path-SGD is rescaling invariant.
\end{proof}

\paragraph{Efficient Implementation:}
The Path-SGD update rule~\eqref{eq:update}, in the way it is written,
needs to consider all the paths, which is exponential in the depth of
the network. However, it can be calculated in a time that is no more
than a forward-backward step on a single data point. That is, in a
mini-batch setting with batch size $B$, if the backpropagation on the
mini-batch can be done in time $BT$, the running time of the Path-SGD
on the mini-batch will be roughly $(B+1)T$ -- a very moderate runtime
increase with typical mini-batch sizes of hundreds or thousands of
points.  Algorithm ~\ref{alg:update} shows an efficient implementation
of the Path-SGD update rule.

We next compare Path-SGD to other optimization methods in both balanced and unbalanced settings.
 
\begin{algorithm}[t]
  \caption{Path-SGD update rule}\label{alg:update}
  \begin{algorithmic}[1]
  \State $\forall_{v\in V^0_{\IN}}\; \gamma_{\IN}(v)=1$\Comment{Initialization}
  \State $\forall_{v\in V^0_\OUT}\; \gamma_{\OUT}(v)=1$
  \For{$i=1\;\textbf{to}\;d\;$}
  \State $\forall_{v\in V^i_{\IN}} \gamma_{\IN}(v) = \sum_{(u\rightarrow v)\in E} \gamma_{\IN}(u)\abs{w_{(u,v)}}^p$
  \State $\forall_{v\in V^i_{\OUT}} \gamma_{\OUT}(v) = \sum_{(v\rightarrow u)\in E} \abs{w_{(v,u)}}^p\gamma_{\OUT}(u)$
  \EndFor
  \State $\forall_{(u\rightarrow v)\in E}\;\; \gamma(w^{(t)},(u,v)) = \gamma_{\IN}(u)^{2/p}\gamma_{\OUT}(v)^{2/p}$
  \State $\forall_{e\in E} w^{(t+1)}_e = w^{(t)}_e - \frac{\eta}{\gamma(w^{(t)},e)} \frac{\partial L}{\partial w_e}(w^{(t)})$\Comment{Update Rule}
   \end{algorithmic}
\end{algorithm}


\section{Experiments on Path-SGD}

We compare $\ell_2$-Path-SGD to two commonly used optimization methods in deep learning, SGD and AdaGrad. We conduct our experiments on 
four common benchmark datasets: the standard MNIST dataset of handwritten digits~\citep{lecun1998gradient}; 
CIFAR-10 and CIFAR-100 datasets of tiny images of natural scenes~~\citep{krizhevsky2009learning}; 
and Street View House Numbers (SVHN) dataset containing 
color images of house numbers collected by Google Street View~\citep{netzer2011reading}. 
Details of the datasets are shown in Table~\ref{table}.

\begin{table}[t]
\caption{General information on datasets used in the experiments on feedforward networks.}
\label{table}
\begin{center}
\begin{tabular}{c c c c c}
{\bf Data Set}  &{\bf Dimensionality}&{\bf Classes}&{\bf Training Set}&{\bf Test Set}
\\ \hline
CIFAR-10&3072 ($32 \times 32$ color)&10&50000&10000\\
CIFAR-100&3072 ($32 \times 32$ color)&100&50000&10000\\
MNIST&784 ($28 \times 28$ grayscale)&10&60000&10000\\
SVHN&3072 ($32 \times 32$ color)&10&73257&26032\\
\hline
\end{tabular}
\end{center}
\end{table}

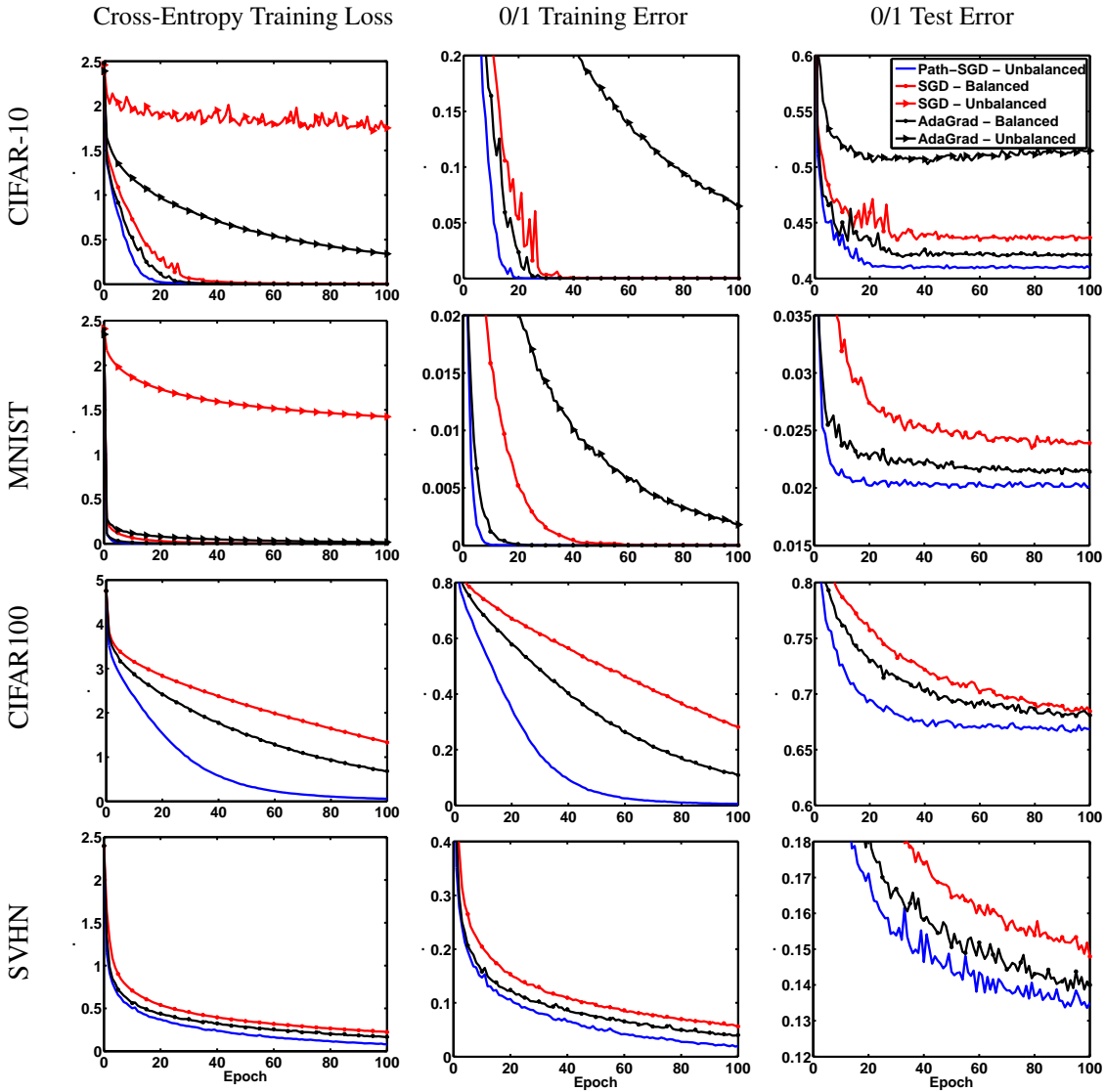
\begin{figure}[t!]
\hspace{0.1in}
 \subfloat{
  \begin{tabular}{r}
   \includegraphics[width=\picwidth]{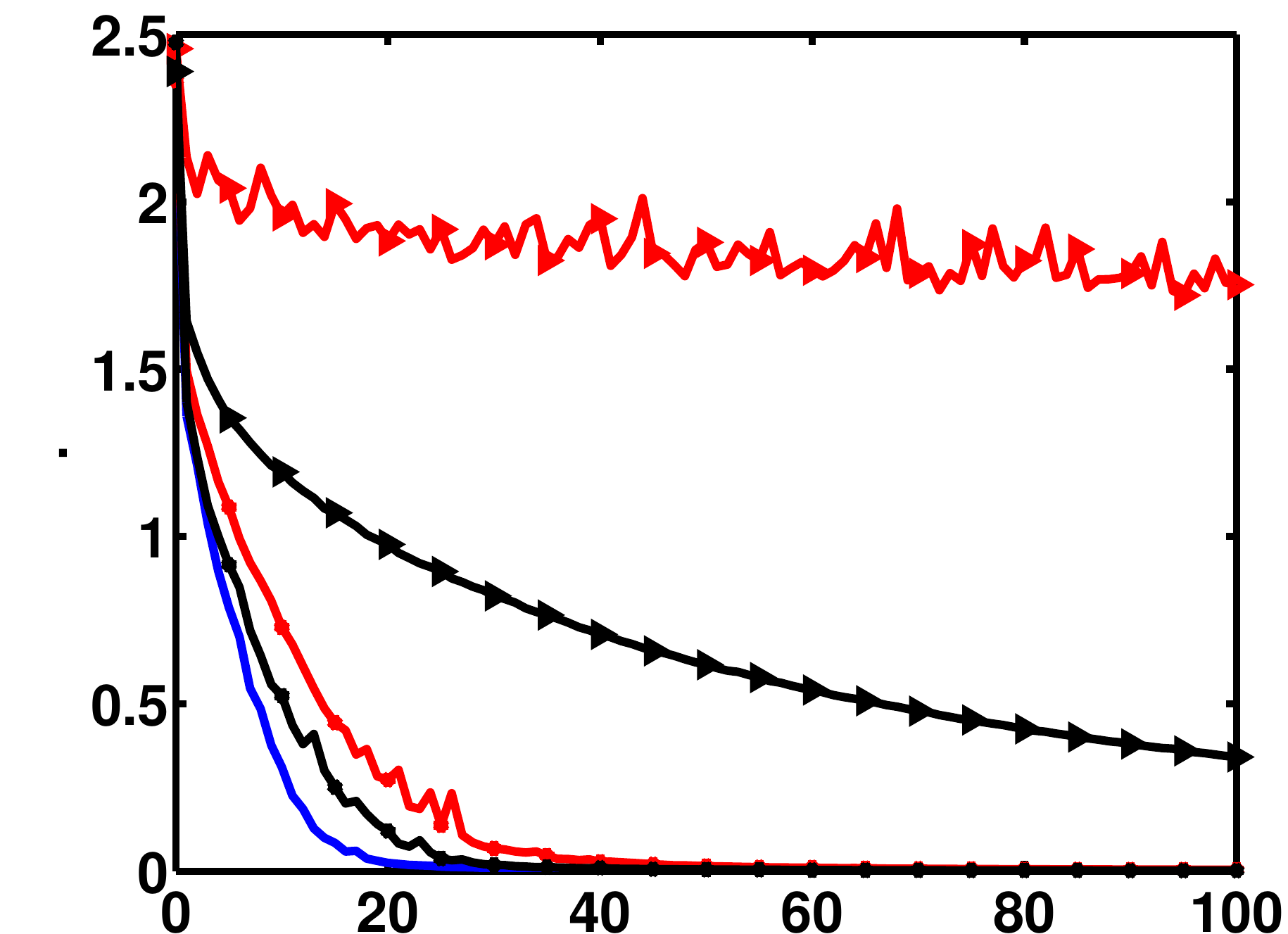} \\
      \includegraphics[width=\picwidth]{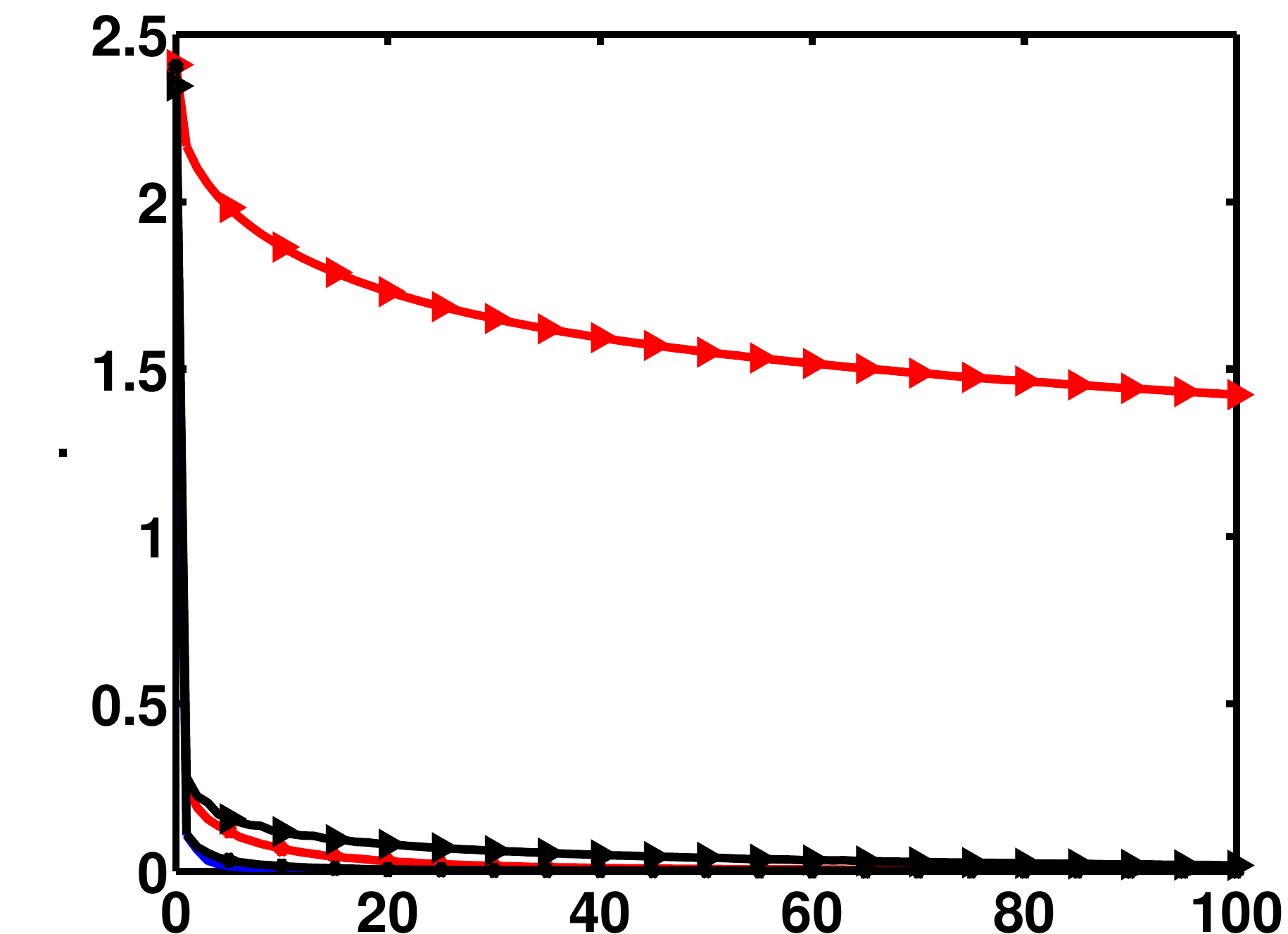} \\
   \includegraphics[width=1.76in]{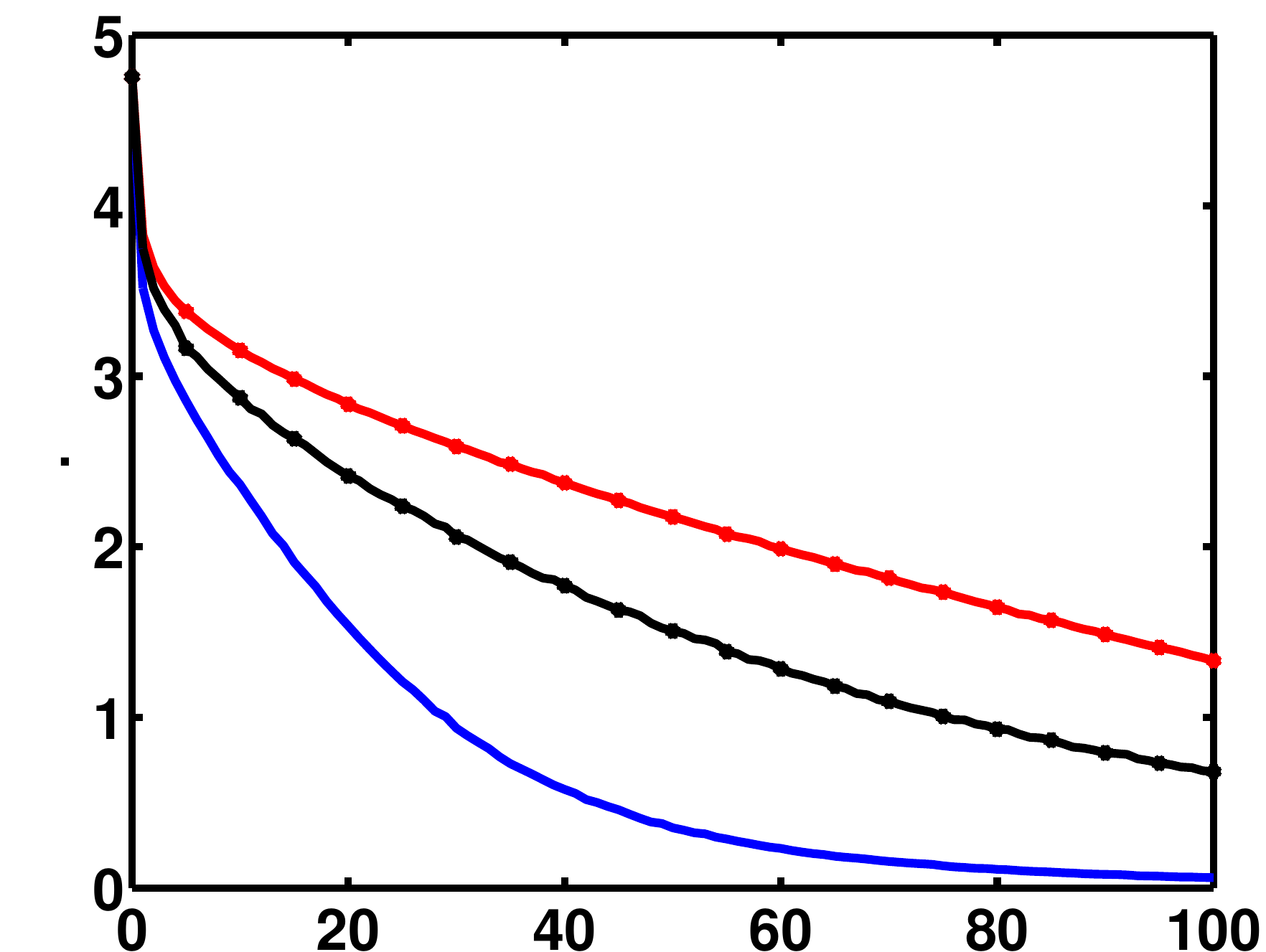} \\
   \includegraphics[width=\picwidth]{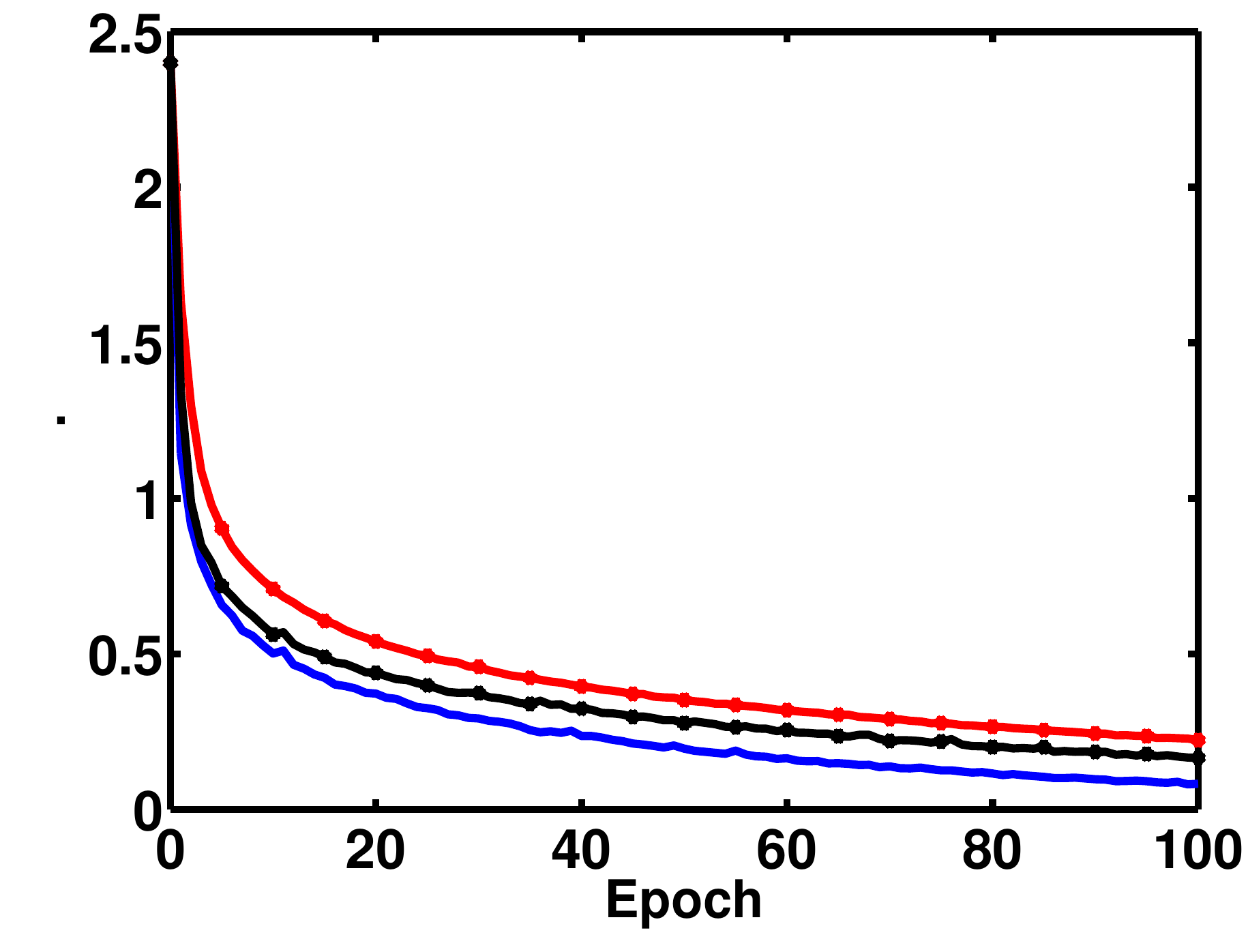}
  \end{tabular}
 }
 \hspace{-0.3in}
 \subfloat{
  \begin{tabular}{r}
   \includegraphics[width=\picwidth]{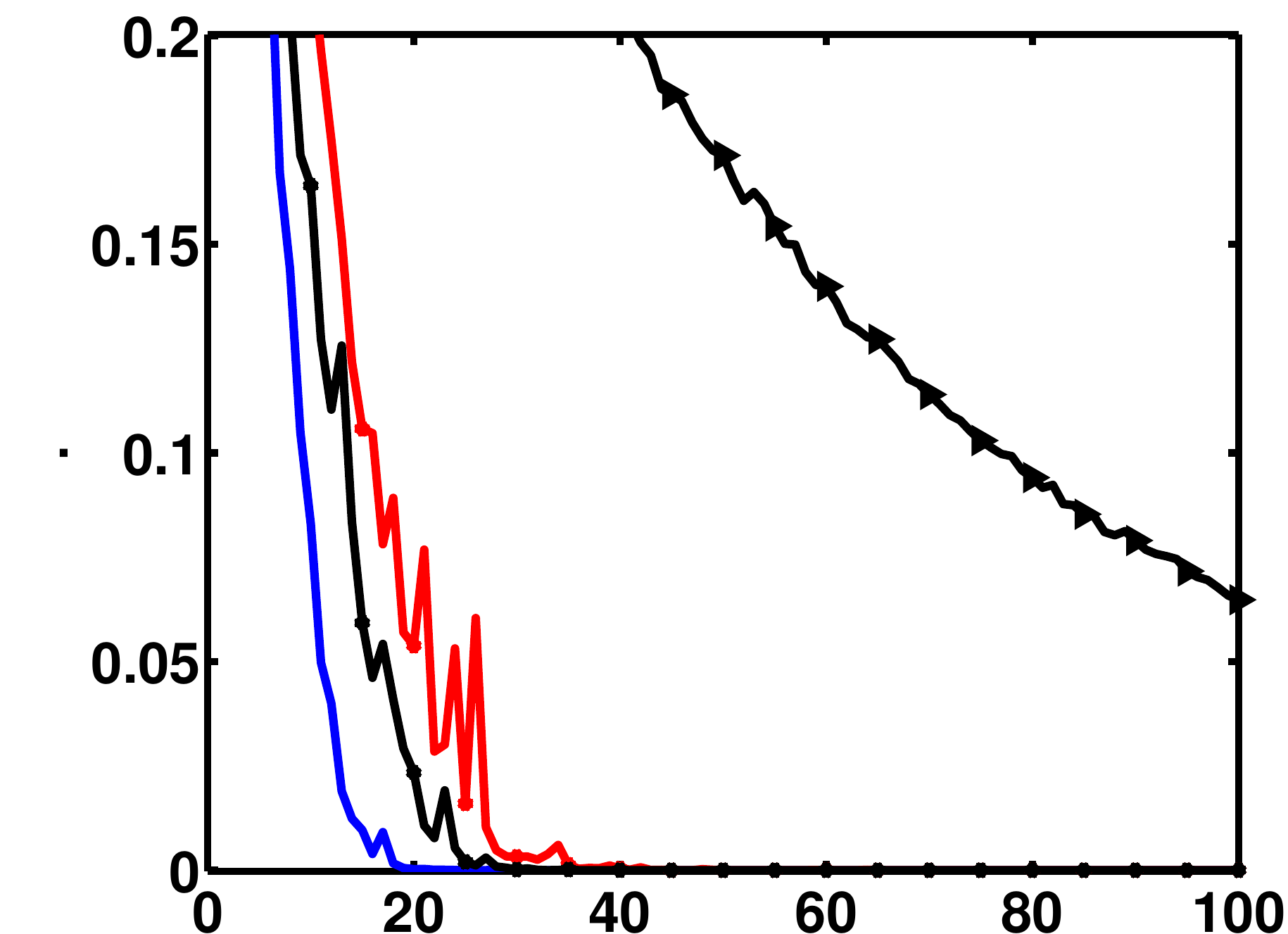} \\
   \includegraphics[width=1.9in]{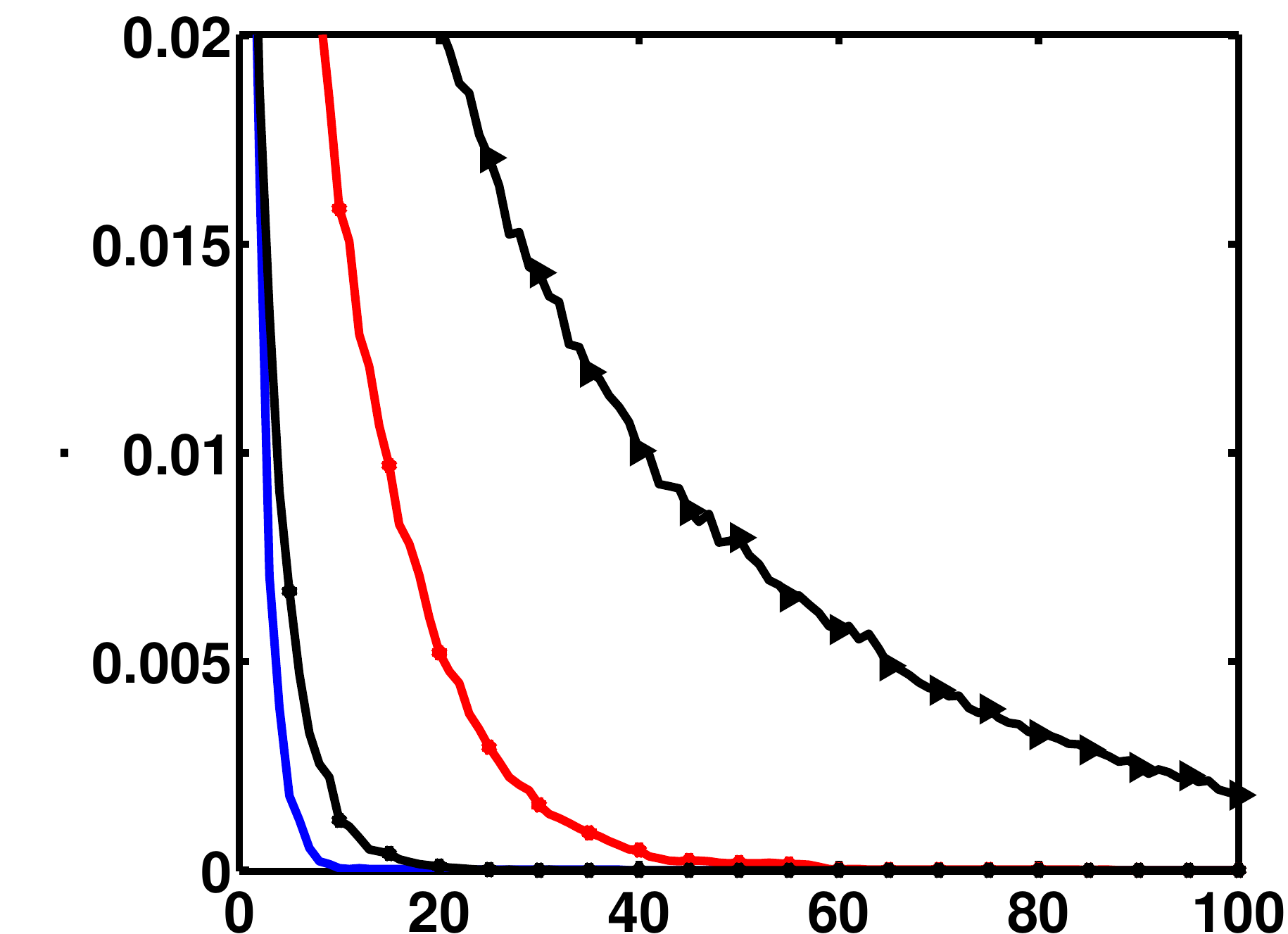} \\
      \includegraphics[width=\picwidth]{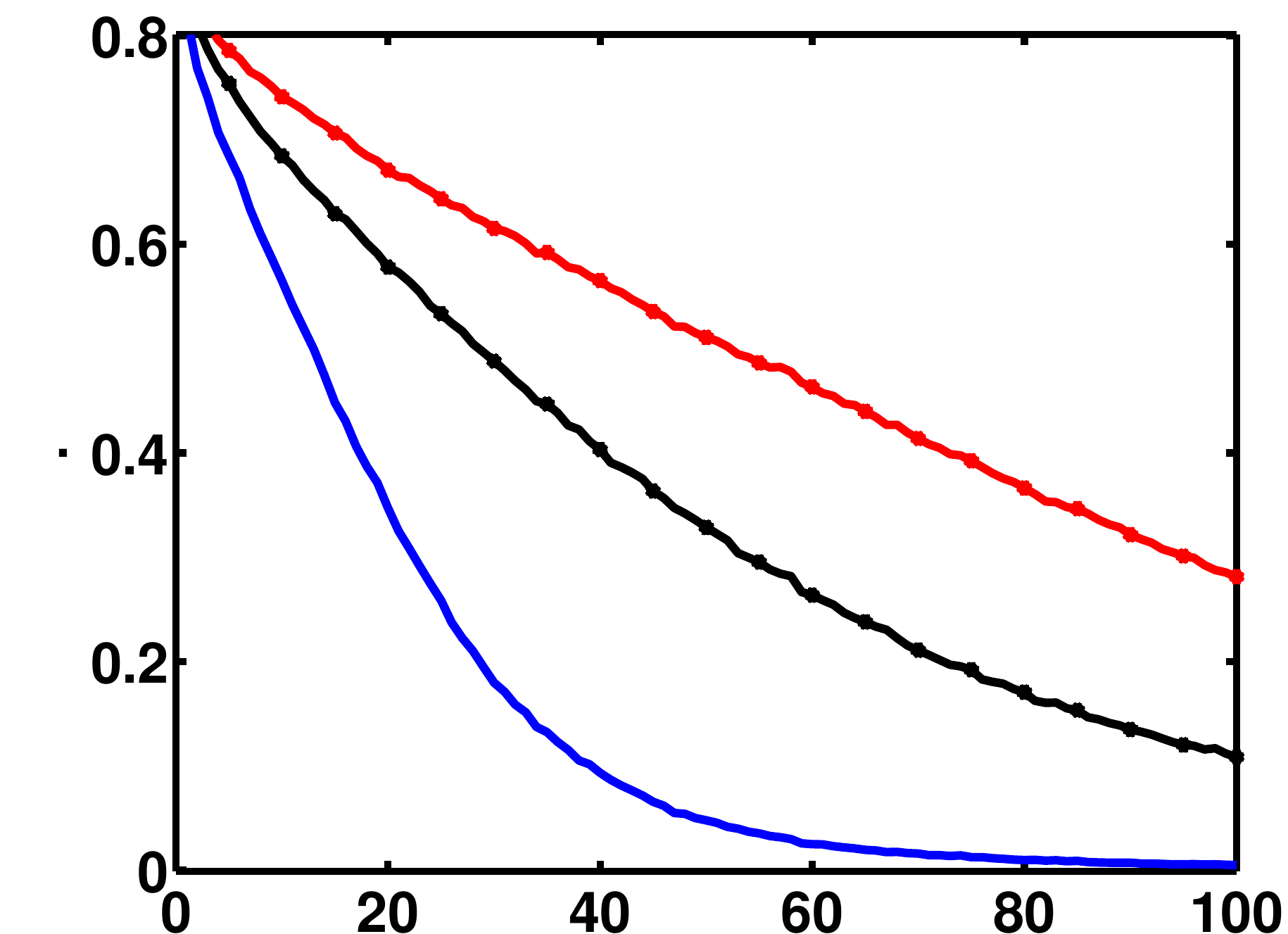} \\
   \includegraphics[width=1.85in]{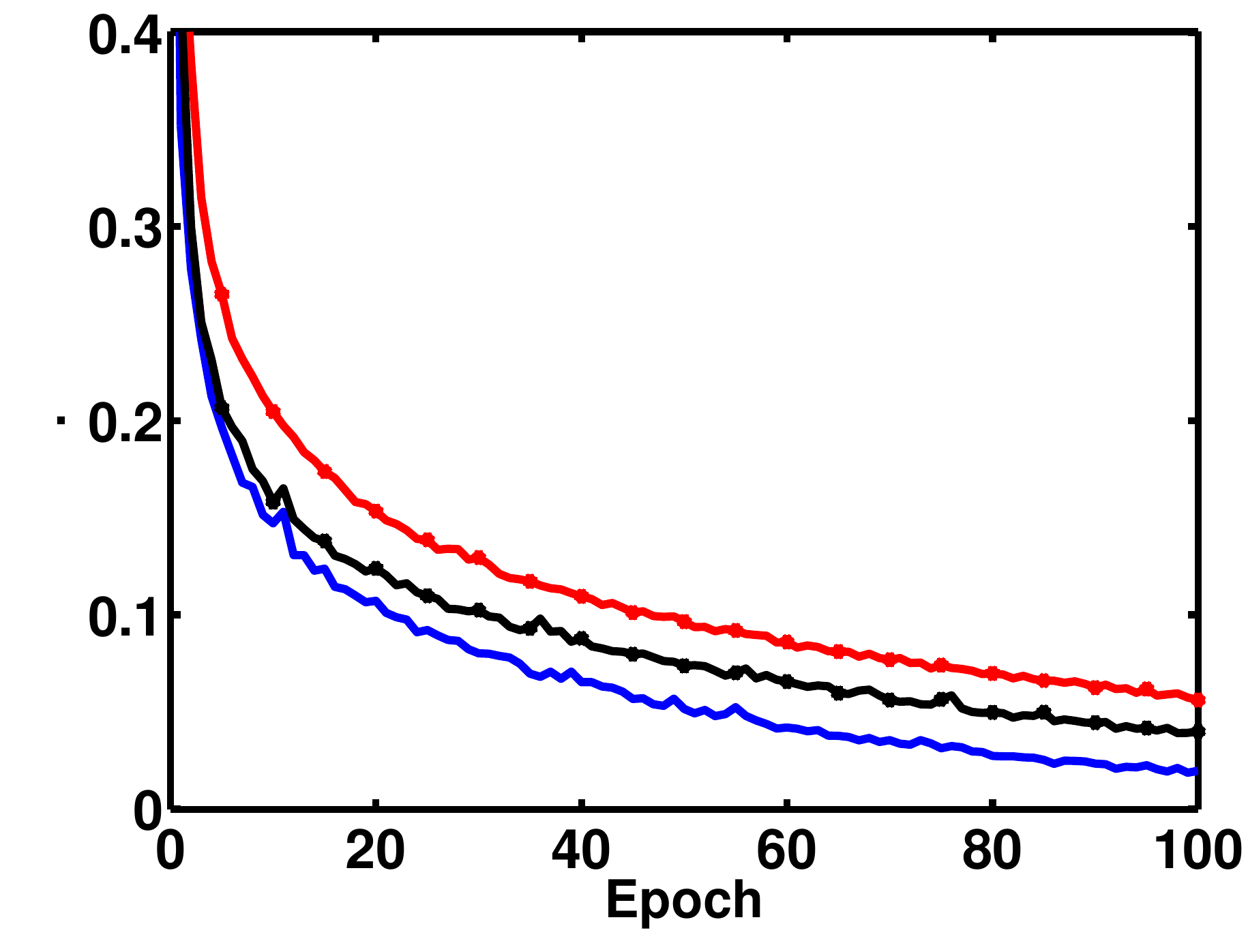}
  \end{tabular}
 }
 \hspace{-0.3in}
 \subfloat{
  \begin{tabular}{r}
   \includegraphics[width=\picwidth]{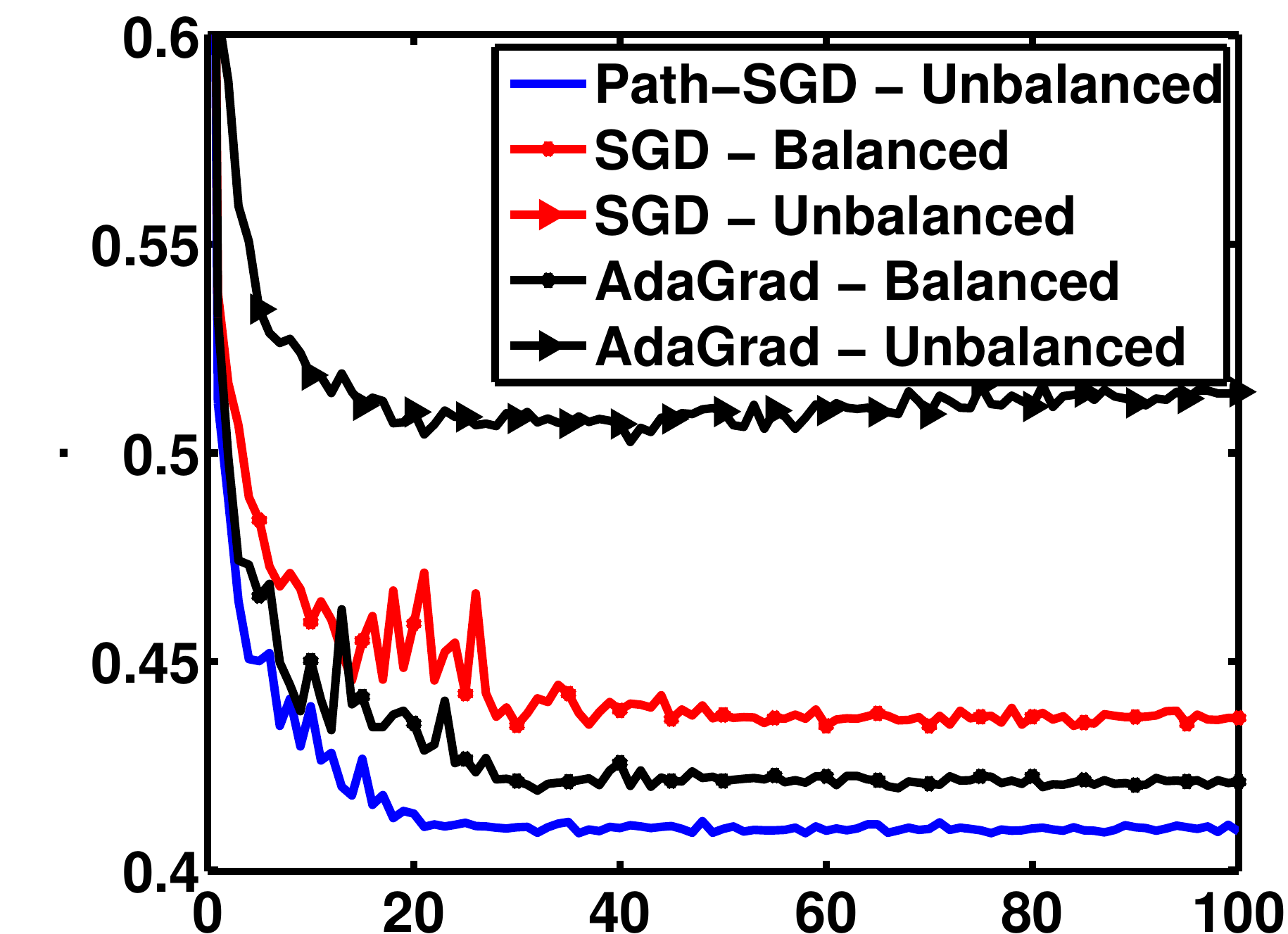} \\
      \includegraphics[width=1.9in]{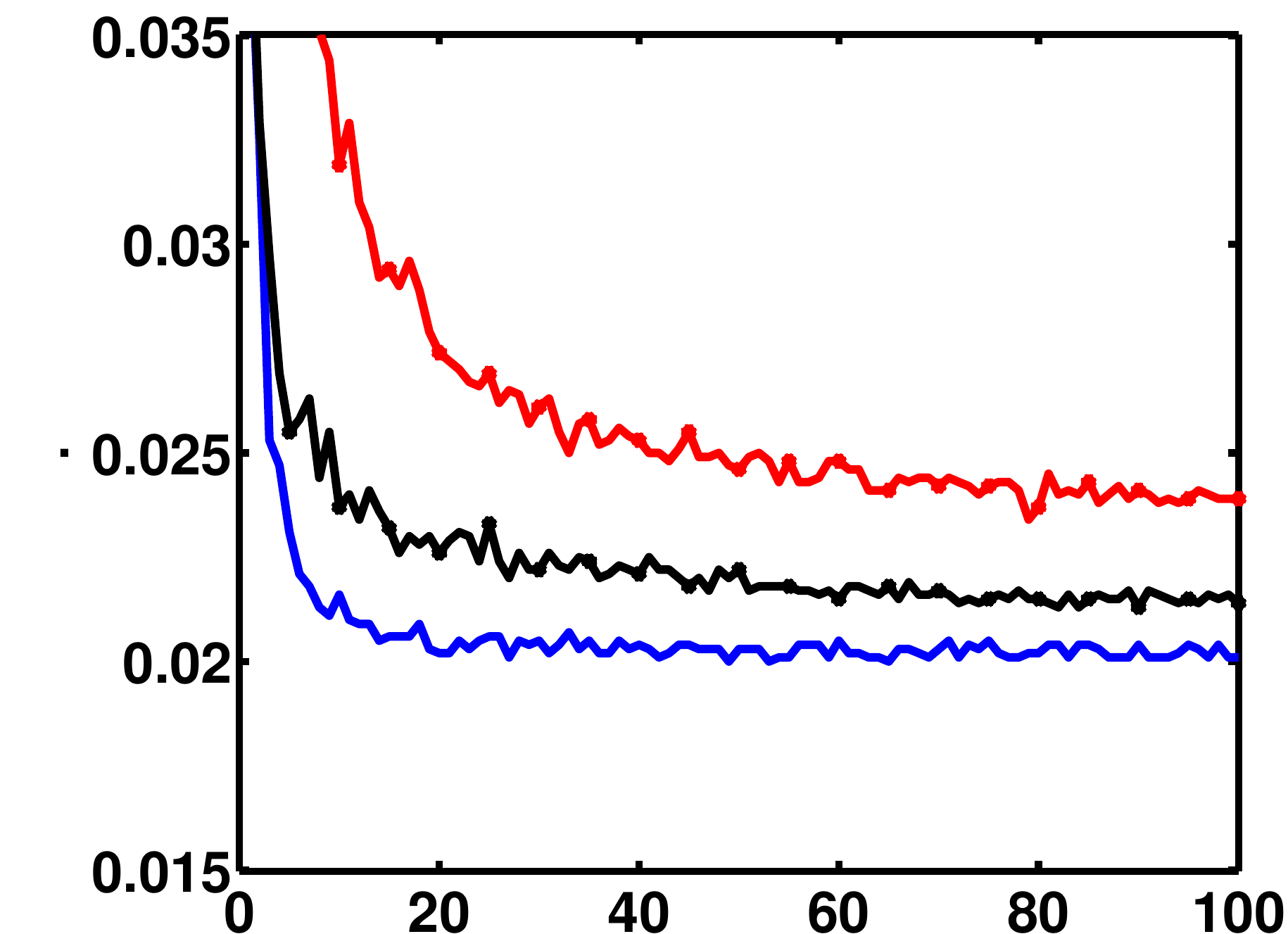} \\
   \includegraphics[width=\picwidth]{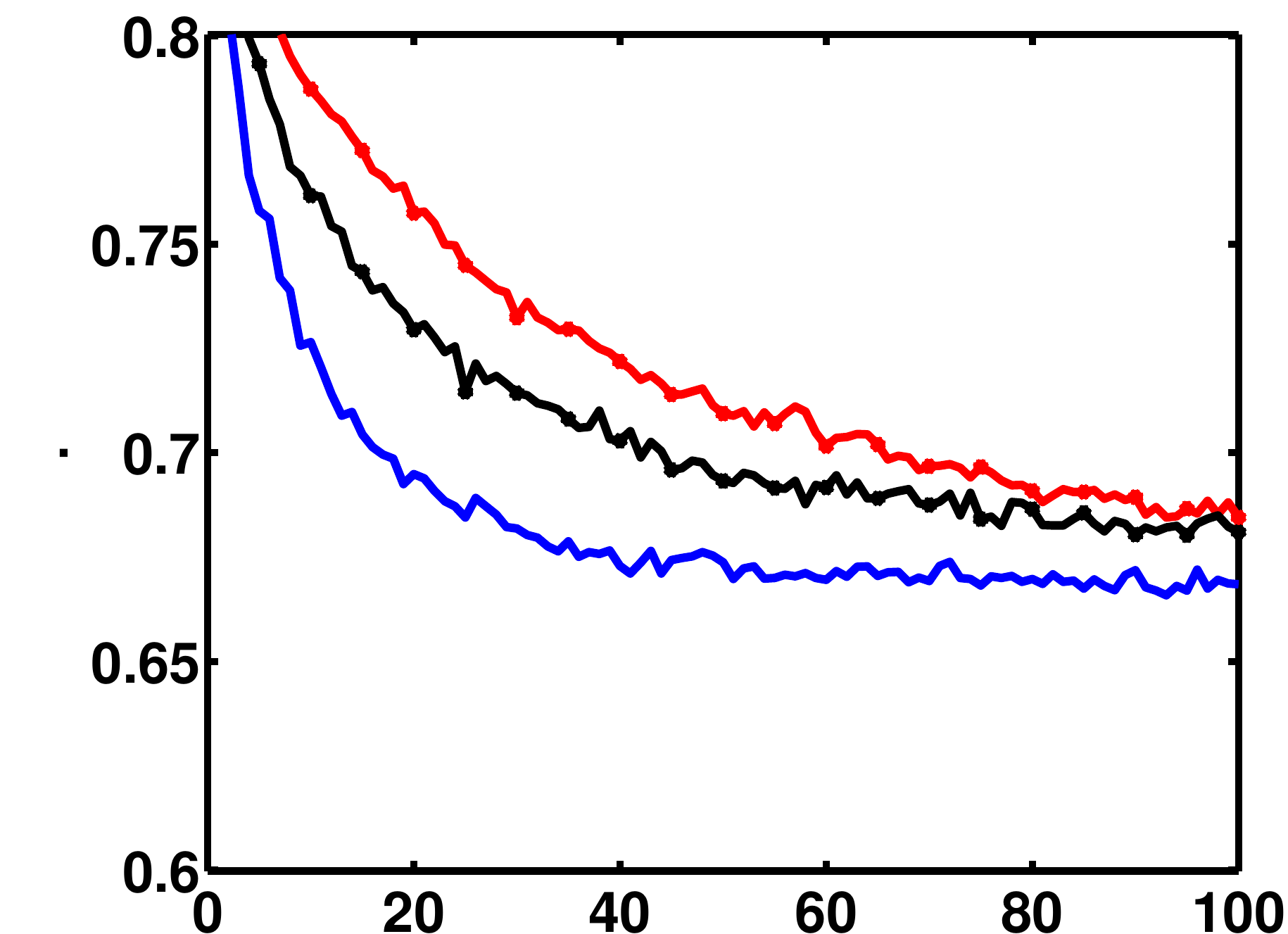} \\
   \includegraphics[width=1.85in]{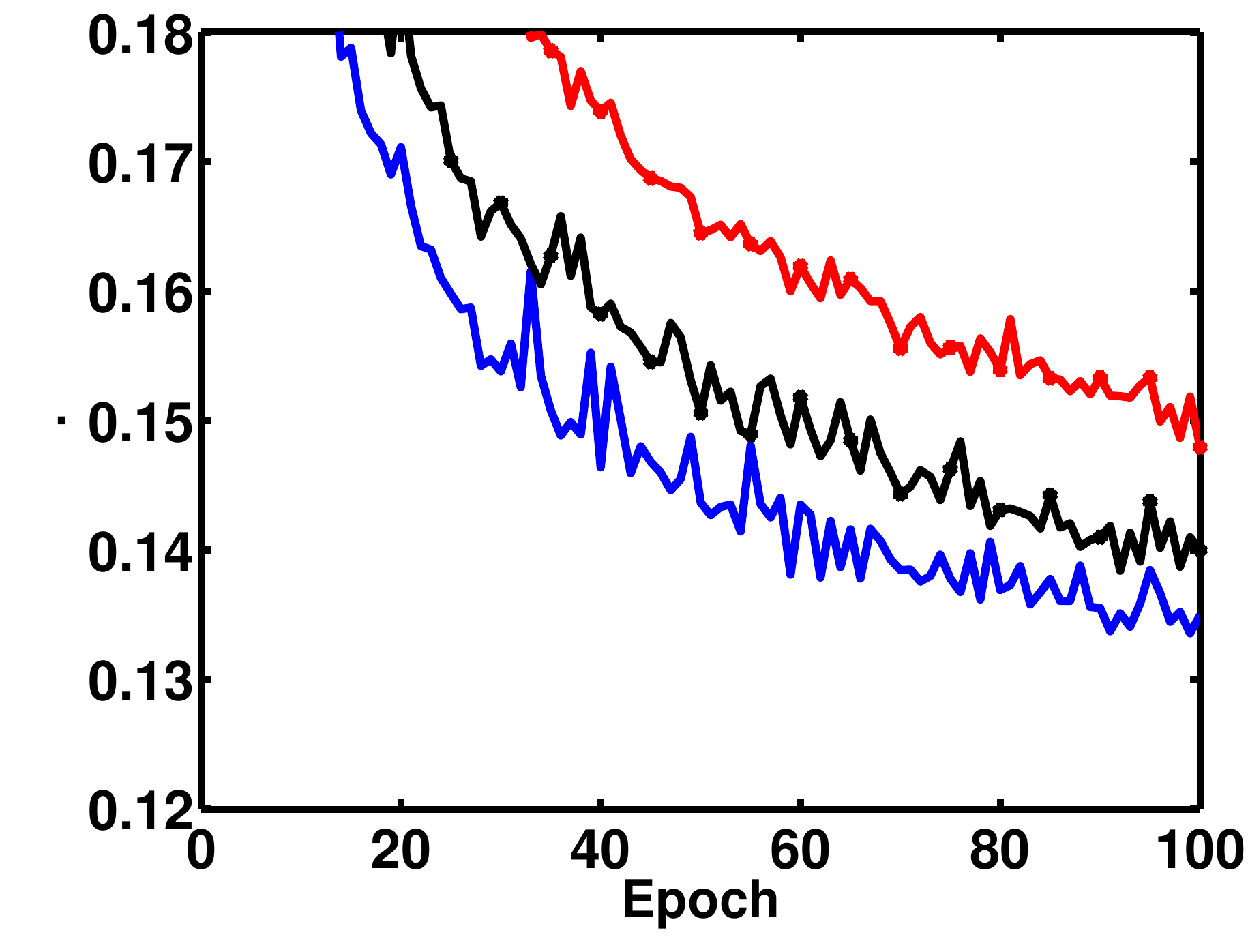}
  \end{tabular}
 }

 \begin{picture}(0,0)(0,0)
\rotatebox{90}{\put(342, 0){CIFAR-10}\put(240, 0){MNIST}\put(147, 0){CIFAR100}\put(50, 0){SVHN}}
\end{picture}
  \begin{picture}(0,0)(0,0)
{\put(30, 420){\small Cross-Entropy Training Loss}\put(187, 420){\small 0/1 Training Error}\put(328, 420){ \small 0/1 Test Error}}
\end{picture}
 \caption{\footnotesize Learning curves using different optimization methods 
 for 4 datasets without dropout. Left panel displays the cross-entropy objective function; 
middle and right panels show the corresponding values of the training and test errors, where the values are reported on
different epochs during the course of optimization.We tried both balanced and unbalanced initializations. In balanced initialization, incoming weights to each unit $v$ are initialized to i.i.d samples from a Gaussian distribution with standard deviation $1/\sqrt{\text{fan-in}(v)}$. In the unbalanced setting, we first initialized the weights to be the same as the balanced weights. We then picked 2000 hidden units randomly with replacement. For each unit, we multiplied its incoming edge and divided its outgoing edge by $10c$, where $c$ was chosen randomly from log-normal distribution. Although we
proved that Path-SGD updates are the same for balanced and unbalanced
initializations, to verify that despite numerical issues they are
indeed identical, we trained Path-SGD with both balanced and unbalanced initializations. 
Since the curves were exactly the same we only show a single curve. Best viewed in color.}
 \label{fig:drop-nodrop}
\vspace{-0.1in}
\end{figure}

In all of our experiments, we trained feed-forward networks with two hidden layers, each containing 4000 hidden units. We used mini-batches of size 100 and the step-size of $10^{-\alpha}$, where $\alpha$ is an integer between 0 and 10. To choose $\alpha$, for each dataset, we considered the validation errors over the validation set (10000 randomly chosen points that are kept out during the initial training) and picked the one that reaches the minimum error faster. We then trained the network over the entire training set. All the networks were trained both with and without dropout. When training with dropout, at each update step, we retained each unit with probability 0.5.

The optimization results are shown in Figure~\ref{fig:drop-nodrop}. For each of the four datasets, the plots for
objective function (cross-entropy), the training error and the test
error are shown from left to right where in each plot the values are
reported on different epochs during the optimization. The dropout is used for the 
experiments on CIFAR-100 and SVHN. Please see \cite{neyshabur2015path} for a more complete set of experimental results.

We can see in Figure~\ref{fig:drop-nodrop} that not only does Path-SGD often get to the same value of objective function, training and test error faster, but also the plots for test errors demonstrate that implicit regularization due to steepest descent with respect to path-regularizer leads to a solution that generalizes better. This provides further evidence on the role of implicit regularization in deep learning.

The results suggest that Path-SGD outperforms SGD and AdaGrad in two
different ways. First, it can achieve the same accuracy much faster
and second, the implicit regularization by Path-SGD leads to a local
minima that can generalize better even when the training error is
zero. This can be better analyzed
by looking at the plots for more number of epochs which we have
provided in  \cite{neyshabur2015path}. We should also point that
Path-SGD can be easily combined with AdaGrad or Adam to take advantage of the
adaptive stepsize or used together with a momentum term. This could
potentially perform even better compare to Path-SGD.

\section{Discussion}

We demonstrated the implicit regularization in deep learning through experiments and discussed
the importance of geometry of optimization in finding a ``low complexity'' solution. 
Based on that, we revisited the choice of the Euclidean geometry on the weights of
RELU networks, suggested an alternative optimization method
approximately corresponding to a different geometry, and showed that
using such an alternative geometry can be beneficial.  In this work we show
proof-of-concept success, and we expect Path-SGD to be beneficial also
in large-scale training for very deep convolutional networks.
Combining Path-SGD with AdaGrad, with momentum or with other
optimization heuristics might further enhance results.

Although we do believe Path-SGD is a very good optimization method,
and is an easy plug-in for SGD, we hope this work will also inspire
others to consider other geometries, other regularizers and
perhaps better, update rules.
A particular property of Path-SGD is its rescaling invariance, which
we argue is appropriate for RELU networks.  But Path-SGD is certainly
not the only rescaling invariant update possible, and other invariant
geometries might be even better.

Finally, we choose to use steepest descent because of its simplicity
of implementation.  A better choice might be mirror descent with
respect to an appropriate potential function, but such a construction
seems particularly challenging considering the non-convexity of neural
networks.

\acks{Research was partially funded by NSF award IIS-1302662 and Intel ICRI-CI.
}


\bibliography{biblio}
\end{document}